\documentclass{article}

\usepackage[final]{gpsmdms_2022}
\usepackage{ant_cmds}

\usepackage[utf8]{inputenc} 
\usepackage[T1]{fontenc}    
\usepackage{hyperref}       
\usepackage{url}            
\usepackage{booktabs}       
\usepackage{amsfonts}       
\usepackage{nicefrac}       
\usepackage{microtype}      
\usepackage{xcolor}         
\usepackage{subfig}

\title{Provably Reliable Large-Scale Sampling from Gaussian Processes}

\author{%
  Anthony Stephenson \\
  Department of Mathematics\\
  Bristol University\\
\And
  Robert Allison \\
  Department of Mathematics\\
  Bristol University\\
\And
    Edward Pyzer-Knapp \\
    IBM Research
}

\begin{document}

\maketitle


\begin{abstract}
  When comparing approximate Gaussian process (GP) models, it can be helpful to be able to generate data from any GP. If we are interested in how approximate methods perform at scale, we may wish to generate very large synthetic datasets to evaluate them. Na\"{i}vely doing so would cost \(\order{n^3}\) flops and \(\order{n^2}\) memory to generate a size \(n\) sample. We demonstrate how to scale such data generation to large \(n\) whilst still providing guarantees that, with high probability, the sample is indistinguishable from a sample from the desired GP.
\end{abstract}

\section{Introduction}

\subsection{Motivation}
    In the GP literature, even when an approximate model designed to be highly scalable is introduced, evaluation is usually done on a small-scale toy synthetic dataset of low dimension and on a selection of real datasets. When large synthetic data are used, they are often noisy versions of deterministic functions rather than genuine samples from a Gaussian process prior. We believe this leaves a gap in the analysis: careful assessment of performance at scale under controlled, model-consistent conditions. As such, in this paper we show how one can generate datasets that grant this option whilst defining the criteria necessary for these datasets to function as benchmarks. We acknowledge the work in \cite{Wilson2020} on sampling from the GP \emph{posterior} but point out that it is unhelpful for our purposes.

\subsection{Sampling from Gaussian Processes}
    To test GP approximations in a controlled environment we want to be able to generate reliable synthetic datasets from arbitrary GPs. Since a draw from a GP evaluated at a finite set of points is distributed according to a multivariate normal  distribution with some mean and covariance $m$ and $C$, we can form a sample by generating points from a standard Normal distribution and performing a linear  transformation using some decomposition of the covariance \(C=AA^T\) i.e. \(u \sim \normdist[u]{0}{I_n}\), \(y = Au + m\) \(\implies\) \(y \sim \normdist[y]{m}{C}\). Although we can generate $u$ efficiently, the decomposition \(C=AA^T\) is expensive (\(\order{n^3}\)) using standard methods (SVD or Cholesky). This quickly
 becomes infeasible at large $n$, much like GP regression. As a result, we look to approximate methods to generate synthetic data at scale. 

\subsection{Outline}
    In this paper we seek to determine how to affordably generate large (approximate) samples from GPs such that any given sample is \emph{indistinguishable} (in some sense) from a sample drawn from the exact GP. 
    We do so by first reviewing GP approximation techniques (excluding conjugate gradient based methods, which are considered in \ref{appx:cg}) from the literature that can be leveraged to generate samples from a GP prior and then computing bounds on relevant parameters to constrain the error between the approximate and exact samples (sections \ref{sec:rff} and \ref{sec:ciq}). We define this error in terms of the total variation distance. We go on to define our notion of ``indistinguishability'' in section \ref{sec:indistinguishable} before running experiments to support our claims in section \ref{sec:experiments}.
    

\section{Random Fourier Features (RFF)}
\label{sec:rff}
  RFF were introduced as a method of approximating kernels at large scales in Support Vector Machines and Kernel Ridge Regression problems in \cite{Rahimi2007}. The method has since been studied extensively, for example in \cite{Li2019a,Yang2012,Sriperumbudur2015,Liu2021,Bach2017,Choromanski2018}. One of the appealing features of the RFF approximation for sampling from a GP is that we don't need to form the full Gram matrix (given by \(ZZ^T\) with \(Z \in \mathbb{R}^{n\times D}\)) to generate samples. To generate samples we need only construct a single $Z$ matrix and transform a variable $w\sim\normdist{0}{I_D}$ to get $\hat{f}=Zw$. We use the formulation suggested in \cite{Sutherland2015} to minimise the variance of our sampling and adapt their proof technique for our purposes in \ref{appx:rff_proof}.

    \begin{lemma}
        \label{lem:rff}
            To generate a sample of size $n$ whose marginal distribution differs from the true marginal distribution from a
            given GP by a total variation distance (\(\TV\)) of at most $\eps$, with probability \(1-\delta\) it is sufficient to use $D$ RFFs, where
                \(
                D \geq 8\log\left(\frac{n}{\sqrt{\delta}}\right)\frac{n^2}{8\eps^2\sigma_\xi^4}
                \)
            for some $\delta > 0$.
    \end{lemma}
     
     This leads to an overall sampling complexity of \(\order{nD}=\order{n^3\log n}\) which
     is \emph{worse} than what we would get using Cholesky decomposition. However, it is worth noting that, in terms of memory usage and ease of parallelism, this method is still competitive since we need only generate a single sample at a time, computing a single vector inner product per sample. With careful implementation, memory cost can be as low as \(\order{1}\) (see \ref{appx:rff} for details).

\section{Contour Integral Quadrature (CIQ)}
\label{sec:ciq}
   There is some literature dedicated to the computation of functions of square matrices via the approximation of the
     Cauchy integral formula (\cite{Davies2005,Hale2008,Pleiss2020}). An algorithm for the function of interest for us, \(A^\half\), is derived in \cite{Hale2008} and subsequently built upon by \cite{Pleiss2020} where the authors derive an
     efficient quadrature algorithm in this and the $A^{-\half}$ case, specifically citing sampling as a potential usage. We make use of their algorithm in this vein to estimate
     matrix-vector products of the form $K^\half u$. We refer the reader to these sources for a thorough explanation of the method involved, but include a brief summary in \ref{appx:ciq_summary}.
     
     In addition to the superior time complexity we show below, this algorithm also has a modest memory overhead (\(\order{Qn}\) with \(Q\) the number of quadrature points) and general application to \emph{any} kernel, unlike the RFF method which is only applicable to stationary kernels and necessitates non-trivial derivations of Fourier features for non-RBF kernels.
     
     \begin{algorithm}[H]
         Define parameters $d$ ($x$-dimension), $\theta$ (kernel parameters), $\sigma_\xi^2$ (noise-variance) and $\eta$
         (weight of noise-variance at CIQ approximation stage).
         \begin{enumerate}
             \item Sample $x$ data from some distribution, e.g. $x\sim\normdist[x]{0}{\frac{1}{d}I_d}$.
             \item Construct partially noisy kernel $K_{\xi,ij}=k(x_i,x_j)+\eta\sigma_\xi^2\delta_{ij}$
              \item Sample $u\sim\normdist[u]{0}{I_n}$.
             \item Use CIQ to approximate $f\approx \hat{f} = M u$ where \(M\approx K_{\eta\xi}^\half\).
             \item Add noise to the sample to get $\hat{y} = \hat{f} + \xi$ with 
             $\xi \sim\normdist{0}{(1-\eta)\sigma_\xi^2 I_n}$.
         \end{enumerate}
         \caption{CIQ Sampling}
         \label{alg:ciq_sample}
     \end{algorithm}
     
     \begin{lemma}
     \label{lem:ciq_bounds}
         To sample approximate draws from a Gaussian Process with \(\TV < \eps\) when compared to a
         draw from the exact Gaussian Process, $Q$ quadrature points and $J$ Lanczos iterations will be sufficient provided we use the
         CIQ procedure from algorithm \ref{alg:ciq_sample} to generate our draw, where $Q$ and $J$ satisfy \(Q \geq \order{\log \left(\frac{n}{\eta\sigma_\xi^2}\right)(-\log\delta_Q)}\) and \(J \geq \ordtilde{\frac{\sqrt{n}}{\sqrt{\eta}\sigma_\xi}\log\frac{n}{\sigma_\xi(\eps\sigma_\xi\sqrt{1-\eta}-\delta_Q)}}\)
        with \(0<\delta_Q<\eps\sigma_\xi\sqrt{1-\eta}\).
     \end{lemma}
     
     \subsection{Preconditioning}
     \label{sec:precon}
         It is shown in the appendix (\ref{eq:epsJ}) that the number of iterations $J$ depends primarily on the condition number of
         the kernel, which implies that we can reduce \(J\) using preconditioning.

         \begin{lemma}
         \label{lem:nystrom_precond}
             Using a rank-$k$ ($k=\sqrt{n}$) Nystr\"{o}m preconditioner on an ($n\times n$) kernel matrix with noise
             variance $\eta\sigma_\xi^2$ and some constant \(\tilde{C}'>0\) means that setting
             \newline
             \(
             J\geq 1 + \frac{\sqrt{\lambda_{k+1}}n^{3/8}}{\sqrt{\eta}\sigma_\xi} \left(\frac{5}{4}\log n-\log(\eps\sigma_\xi\sqrt{1-\eta}-\delta_Q) +\tilde{C}'\right)
             \)
             Lanczos iterations in the CIQ algorithm will satisfy our requirements.
         \end{lemma}

         

        To make \ref{lem:nystrom_precond} more useful, we rely on a result from \ref{lem:general_precond} that shows that for a certain class of kernels we can relate the \(k+1^{th}\) eigenvalue to \(n\), to get a workable bound on \(J\):
         \begin{lemma}
         \label{lem:general_precond}
             For a sufficiently smooth radial (\ref{proof:general_precond}) kernel function $k\in L^2_\mu$, some constant \(c_1>0\) and Nyström preconditioner of rank $\floor{\sqrt{n}}$ we define a variable \(\gamma = \frac{7}{8}\log n - \frac{c_1}{2}n^{1/d}\) to obtain sufficient conditions for \(J\) under three possible scenarios:
             \begin{enumerate}[label=\roman*]
                 \item \label{itm:modn} Moderate \(n\) s.t. \(\gamma > 1\): \( J \geq \order{n^{7/8}\log n} \)
                 \item \label{itm:bign}Larger \(n\) s.t. \(\gamma \in (0,1)\): \(J \geq \order{(\log n)^2}\)
                 \item \label{itm:asympn} \(n\rightarrow\infty\) s.t. \(\gamma < 0\): \(J \geq \order{1}\)
             \end{enumerate}
         \end{lemma}
         
        
\section{``Indistinguishable'' distributions}
\label{sec:indistinguishable}
    We now define what we mean by `indistinguishable'. Assume that samples are provided either from the true GP $P_0$ with $p = \half$ or from the distribution of the approximating method  $P_1$ with $p = \half$ (i.e. a uniform prior on models). Our decision process is to select the model with the largest (exact) posterior probability. This can be shown to produce the smallest achievable error rate (\ref{proof:min_err_rate}). If the models were completely indistinguishable then the error rate would be $\half$. Perfect indistinguishability is an unachievable goal due to limited compute-resource so we instead require \(\Pr(\mbox{error})\) to be within \(\eps\) of \(\half\) for suitably small $\eps$.
    
    \begin{definition}[\(\eps\)-indistinguishable]
        $P_0$ and $P_1$ are {\it $\eps$-indistinguishable}
        if the above optimal Bayesian decision process has $\Pr(\mbox{error}) \geq \half - \eps$.
    \end{definition}
    
    \begin{lemma}
        \label{lem:indistinguishable}
        $P_0$ and $P_1$ are $\eps$-indistinguishable
        if $\TV(P_0,P_1) \leq 2 \eps $.
    \end{lemma}
    
    When combining \ref{lem:indistinguishable} with \ref{lem:rff} and \ref{lem:ciq_bounds}, \ref{lem:nystrom_precond} and \ref{lem:general_precond} and by setting \(\eps\), we can obtain rigorous and stringent guarantees that synthetic data will behave like exact-GP data during subsequent analysis -  in particular for the purpose of evaluating approximate-GP regression performance. A further justification of this notion of indistinguishability and its relation to hypothesis testing is given in \ref{appx:indistinguishable}.

\section{Experiments}
\label{sec:experiments}
    The results above provide bounds on fidelity parameters (\(D,J\)) of the sampling approximations. We run suboptimal hypothesis tests to demonstrate where choices of \(D,J\) are definitely insufficient to reach `indistinguishability' and to enable a like-for-like comparison between CIQ and RFF. The experiments we run generate data from a known GP (with an isotropic RBF kernel) using the approximate sampling procedure being tested. The data are then ``whitened'' using the true kernel matrix such that, when the data is sufficiently close to the true generating GP distribution, the output will be a vector of \(\normdist{0}{1}\) distributed points. We therefore test the hypothesis that the data is from the true GP by running a Cramér-von Mises test for normality on the output at a significance level \(\alpha\) (more detail on the specifics is provided in \ref{appx:experiments}).
    
    We generate a series of M datasets of sizes \((2^m)_{m=1}^M\) over a range of what we consider to be the data hyperparameters; that is the kernel-scale \(\sigma_f^2\), noise variance \(\sigma_\xi^2\), lengthscale \(l\) and dataset dimension \(d\). For each of these setups we run experiments with varying fidelity parameter to determine the value required for the rejection rate from the experiments to converge on the type I error rate, \(\alpha\), implying that the data is indistinguishable from the true GP using the (suboptimal) hypothesis test.

\subsection{Results}
\label{sec:results}
    Figure \ref{main:a} shows the results of our experiments using the RFF method. We see that for each choice of \(n\) and for both lengthscales tested, the rate appears to have converged to the significance level before the number of RFFs predicted by \ref{lem:rff}. We suspect that the requirement that all elements of the difference matrix are bounded (see \ref{appx:rff_proof}) is too stringent and could be relaxed, on the grounds that there are only \(n\), not \(n^2\), independent elements. 
    
    \ref{main:b} and \ref{main:c} show the results when we use the CIQ method without and with preconditioning (respectively). As with the RFF experiments, we see convergence before the theorised bounds, as we should hope. It is clear that preconditioning improves the rate of convergence, as expected. 

\begin{figure}
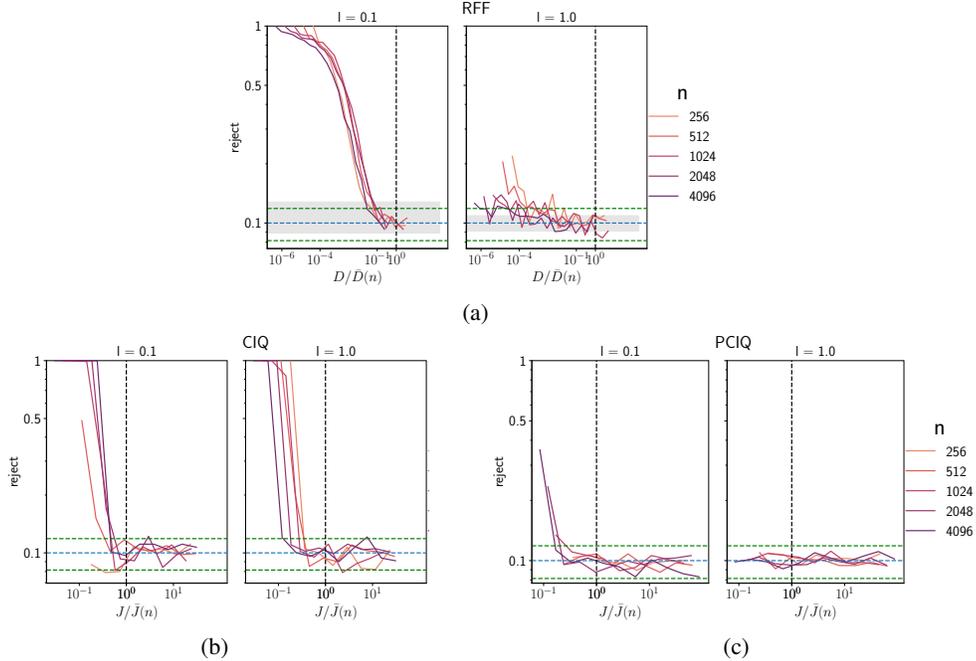

    \centering
    \subfloat[]{\label{main:a}\includegraphics[scale=0.3]{\detokenize{logreject-logD_byN_rff_1_2580211_rescaled}}}\\
    
    \begin{minipage}{.49\linewidth}
    \centering
    \subfloat[]{\label{main:b}\includegraphics[trim={0 0 3.55cm 0},clip,scale=0.3]{\detokenize{logreject-logD_byN_ciq_1_2703849_rescaled}}}
    \end{minipage}
    \begin{minipage}{.49\linewidth}
    \centering
    \subfloat[]{\label{main:c}\includegraphics[scale=0.3]{\detokenize{logreject-logD_byN_ciq_1_2686645_rescaled}}}
    \end{minipage}
    \caption{Rejection rate convergence with size of fidelity parameter. Significance level (\(\alpha\)) is shown by a blue dashed line and the 95\% CI around \(\alpha\) (for converged results) is in green. The range of results obtained from running a Cholesky benchmark is shown by the grey bar. The fidelity parameter is rescaled on the \(x\)-axis by the upper bound derived in the relevant section of this paper. Vertical black dashed line is at \(1.0\) indicating where we reach that bound. (a) shows the RFF case with no. RFF, D and \(\bar{D}(n)=n^2\log n\). (b) shows the convergence with Lanczos iterations \(J\) and \(\bar{J}(n)=\sqrt{n}\log n\). (c) is preconditioned CIQ with \(\bar{J}(n)=n^{3/8}\log n\).}
    \label{fig:results}
\end{figure}

\section{Discussion and conclusion}
    We show how to generate approximate samples from \emph{any} Gaussian Process that, with high probability, cannot be distinguished from a draw from the assumed GP. Bounds are derived to ensure that relevant approximation parameters are chosen to satisfy the requirements on the fidelity of the sample for arbitrary probabilistic bounds at a cost cheaper than a standard approach. We believe this work to be of use to researchers aiming to develop GP approximations for use on large datasets. For practical use, we generally suggest the use of CIQ over RFF or other common approaches on the basis of the strong theoretical guarantees we can provide, given some computational budget. We do note, however, that when memory is the main bottleneck RFF may be a preferable choice. We provide a table in \ref{appx:results} summarising the performance of different algorithms and discuss future work in \ref{appx:further}.

\pagebreak
\section*{Acknowledgments}
We would like to thank Hamza Alawiye for his original implementation of the RFF method; Nick Baskerville and Adam Lee for their advice on aspects of the GPytorch and CIQ code; IBM Research for supplying iCase funding for Anthony Stephenson; NCSC for contributing toward Robert Allison's funding and finally the reviewers for their constructive comments.

\medskip

\bibliographystyle{unsrt} 
\bibliography{references.bib}
\nocite{GP_book}

\appendix

\section{``Indistinguishable'' distributions}
\label{appx:indistinguishable}

      We adopt the definition of \emph{total variation distance} used in \cite{Romano2005}:
      \begin{definition}(Total Variation Distance (TV))
      \label{def:tvd}
          \eq{
              \TV(P_0, P_1) &= \half\norm{P_1-P_0}_1 = \half\int |p_1-p_0|d\mu
          }
          where $p_i$ is the density of $P_i$ with respect to any measure $\mu$ dominating both $P_0$ and $P_1$, i.e. $p_i =
          \frac{dP_i}{d\mu}$.
      \end{definition}
      Using this definition, \(\TV(P_0,P_1)\in [0,1]\) for any measures \(P_0,P_1\). Note that we use notation such that \(\TV(P_0,P_1)= \TV(p_0,p_1)\).

    The two proofs below draw heavily from the proof of Theorem 13.1.1 from \cite{Romano2005}:

\begin{proof}[Proof that the decision process of section \ref{sec:indistinguishable} achieves the minimum possible error rate]
\label{proof:min_err_rate}

    A general decision process for our problem is represented by a function \(\phi : \mathbb{R}^n\rightarrow \{0,1\}\) whereby, if we are presented with a sample vector \(y\ \in \mathbb{R}^n\), we select model \(P_{\phi(y)}\). Under this decision process
    \(\Pr(\mbox{error}) = \Pr(\mbox{select }P_1\given P_0)\Pr(P_0) + \Pr(\mbox{select }P_0\given P_1)\Pr(P_1)\). With an assumed uniform prior on models, we have
    \eq{
        2\Pr(\mbox{error}) &= \Pr(\mbox{select }P_1\given P_0) + \Pr(\mbox{select }P_0\given P_1) \\ 
        &= \int\phi(x)p_0d\mu(x) + \int(1-\phi(x))p_1d\mu(x) \\
        &= 1 + \int \phi(x)(p_0(x)-p_1(x))d\mu(x) \\
        &= 1 + \int_{R_+}\phi(x)f(x)d\mu(x) + \int_{R_-}\phi(x)f(x)d\mu(x) + \int_{R_0}\phi(x)f(x)d\mu(x)
    }
    where \(f(x)=p_0(x)-p_1(x)\), \(R_+=\{x \in \mathbb{R}^n : f(x) > 0\}\), \(R_-=\{x \in \mathbb{R}^n : f(x) < 0\}\) and \(R_0=\{x \in \mathbb{R}^n : f(x)=0\}\). From this it follows that setting \(\phi^\star\) to be 0 on \(R_+\), 1 on \(R_-\) and either 1 or 0 on \(R_0\) minimises the probability of error in the decision process. This choice of \(\phi\) precisely agrees with the decision process described in \cite{Romano2005}.

\end{proof}

\begin{proof}[Proof of \ref{lem:indistinguishable}]
    From the above proof we see that the probability of error achieved by the optimal decision process is \(p'=\half\left[1+\int_{R_-}(p_0(x)-p_1(x))d\mu(x)\right]\). If we interchange the roles of \(p_0\) and \(p_1\) we obtain the alternative representation \(p''=\half\left[1 + \int_{R_+}(p_1(x)-p_0(x))d\mu(x)\right]\). Since the probability of error, \(p\), satisfies \(p=p'=p''\) we have \(p=\half[p'+p'']\) which can be rearranged to give
    \eq{
    p &= \half\left[1 - \half\int\abs{p_1(x)-p_0(x)}d\mu(x)\right] = \half - \frac{1}{2}\TV(P_0,P_1).
    }
    Hence if we can set \(\TV(P_0,P_1)\leq 2\eps\) we obtain a maximin error rate of at least \(\half-\eps\) as claimed.
\end{proof}

\begin{remark}
        Note that we can consider the \emph{worst} decision rule, \(\phi^\dagger\), to do the exact opposite of \(\phi^\star\) and obtain that for any \(\phi\): \(\half - \frac{1}{2}\tau \leq p \leq \half + \frac{1}{2}\tau\) where \(\tau = \TV(P_0,P_1)\). \(p=\half\) only when \(\tau=0\). Additionally, a best-case rate of 0 (under \(\phi^\star\)) is achieved when \(P_0\) and \(P_1\) are maximally different in the sense that \(\tau=1\). In this regime we correspondingly obtain the worst-case rate (under \(\phi^\dagger\)) of 1.
\end{remark}

\subsection{Bounding the total variation distance}

      We could try to directly bound the total variation distance, for example, by using the work in \cite{Devroye2018}.
      However we make use of Pinsker's inequality \[\TV(P_0,P_1) \leq \sqrt{\half\KL(P_0,P_1)}\label{pinsker}\tag{\(\dagger\)}\] and bound the \(\KL\) instead.
      
      Rather than bound the KL-divergence between the marginal distributions (\ref{def:marg_kl}), we seek to bound the \emph{conditional} KL (\ref{def:cond_kl}), since for Gaussian distributions this collapses to the 2-norm of the difference of the sample vectors, \(f,\hat{f}\) (\ref{subsec:gauss}). Details justifying this are provided in \ref{subsec:cond_kl}.    
      

      We define the \emph{marginal} KL divergence as the usual KL divergence, to clearly distinguish from the
      \emph{conditional} KL we define below. Note that we assume \(\KL(Q,P)=\KL(q,p)\) where \(q,p\) are the densities of \(Q,P\) respectively w.r.t some dominating measure \(\mu\).
      \begin{definition}[Marginal Kullback-Leibler Divergence]
          \label{def:marg_kl}
          \eq{
              \KL(q,p) &= \int q(y)\log \frac{q(y)}{p(y)}dy. 
          }
      \end{definition}

      \subsection{Conditional $\KL$}
      \label{subsec:cond_kl}
          First, note that $f$ and $\hat{f}$ are correlated RVs that actually depend on an underlying RV $u\sim
              \normdist{0}{I_n}$. We then condition both the approximate and true distributions ($q$ and $p$) on $u$ rather than
          $f$ and $\hat{f}$. In particular, $f=K^\half u$, $\hat{f}=\hat{K}^\half u$.
          \begin{definition}[Conditional Kullback-Leibler Divergence]
              \label{def:cond_kl}
              \eq{
                  \KL(q,p\given u) &= \int q(y\given u)\log \frac{q(y\given u)}{p(y\given u)}dy. 
              }
          \end{definition}
          
        \begin{lemma}
          \label{lem:cond_kl}
              Using \ref{def:cond_kl},
              \eq{\E_u[\KL(q,p\given u)] &\geq \KL(q,p)}
         \end{lemma}

          \begin{proof}[Proof of \ref{lem:cond_kl}]
              \eq{
                  \E_u[\KL(q,p\given u)] &= \int \pi(u)q(y\given u)\log \frac{q(y\given u)}{p(y\given u)}dydu \\
                  &= \int q(y,u)(\log q(y\given u) - \log p(y\given u))dydu \\
                  &= \int q(y,u)\left[\log\frac{q(y,u)}{\pi(u)} - \log\frac{p(y,u)}{\pi(u)}\right]dydu \\
                  &= \underbrace{\int q(y,u)\log\frac{q(y,u)}{p(y,u)}dydu}_{\KL(q(y,u),p(y,u))} \\
                  &=\int q(u\given y)q(y)\left[\log\frac{q(y)}{p(y)} + \log\frac{q(u\given y)}{p(u\given y)}\right]dydu\\
                  &= \underbrace{\int q(y)\log\frac{q(y)}{p(y)}dy}_{\KL(q,p)} + \int q(y)\underbrace{\int q(u\given
                      y)\log\frac{q(u\given y)}{p(u\given y)}du}_{\KL(q(u\given y), p(u\given y))}dy \\
                  &= \KL(q,p) + \underbrace{q(y')\KL(q(u\given y),p(u\given y))}_{\geq 0}.
              }
              where we have used Tonelli's theorem to reorder the integral, the non-negativity of the KL divergence and subsequently the Mean Value Theorem to complete the proof.
          \end{proof}

      \subsection{Gaussians}
      \label{subsec:gauss}
          Since we are specifically interested in samples from GPs, we are fortunate enough to be dealing with Gaussian
          distributions and, as such, we can write down a closed form for the KL-divergence between two Gaussian measures.

          For the marginal case, with $p(y) = \normdist[y]{0}{K_\xi}$ and $q(y) =
              \normdist[y]{0}{\hat{K}_\xi}$, we have
          \eq{
              \KL(q,p) &= \half\left\{\Tr K_\xi^{-1}\hat{K}_\xi - n
              + \log|K_\xi| - \log|\hat{K}_\xi|\right\}.
          }
          
          We now define \(E=\hat{K}_\xi - K_\xi\) and set $\Delta
         = K_\xi^{-1}E$. Note that $E$ is symmetric but not p.s.d. So we have:
        
        \begin{lemma}
         \label{lem:kl_frob}
         Using the definitions above,
         \eq{
         \KL(q,p) &\leq \frac{\norm{E}_F^2}{4\sigma_\xi^4}. 
         }
        \end{lemma}
        
        \begin{proof} 
        \eq{ 
         \KL(q,p) &= \half\left\{\Tr\Delta - \log|I+\Delta|\right\} \\
         &\leq \frac{1}{4}\Tr\Delta^2 \\
         &\leq \frac{1}{4}\norm{K_\xi^{-1}}_2^2\norm{E}_F^2 \\
         &= \frac{1}{4}\lambda_n(K_\xi)^{-2}\norm{E}_F^2 \\
         &\leq \frac{\norm{E}_F^2}{4\sigma_\xi^4} } 
        where we have used that 
        \eq{ \log|I+\Delta| &=
         \sum_i\log(1+\lambda_i(\Delta)) \geq \sum_i\lambda_i(\Delta) - \half\lambda_i(\Delta)^2 \\
         &\geq \Tr\Delta - \half\Tr\Delta^2,
     }
     and 
     \eq{
     \Tr\Delta^2 &= \Tr[(K_\xi^{-\half}EK_\xi^{-\half})^2] \\
     &= \Tr[(\Lambda_\xi^{-\half}U^TEU\Lambda_\xi^{-\half})^2] \\
     &\leq \lambda_1(K_\xi^{-\half})^2\Tr[EU\Lambda_\xi^{-1}U^TE] \\
     &= \lambda_1(K_\xi^{-1})\Tr[\Lambda_\xi^{-1}U^TE^2U] \\
     &\leq \lambda_1(K_\xi^{-1})^2\Tr(E^2) \\
     &= \norm{K_\xi^{-1}}_2^2\norm{E}_F^2
     }
     with the eigendecomposition \(K_\xi=U\Lambda_\xi U^T\).
     
     
     \end{proof}

          For the conditional distributions $p(y\given f) = \normdist[y]{f}{(1-\eta)\sigma_\xi^2I_n}$ and
          $q(y\given \hat{f})=\normdist[y]{\hat{f}}{(1-\eta)\sigma_\xi^2 I_n}$, we get
        \eq{
            \label{eq:cond_kl_gauss}
            \KL(q,p\given f,\hat{f}) &= \half\left[\Tr I_n + (\hat{f}-f)^T\sigma_\xi^{-2}(1-\eta)^{-1}I_n(\hat{f}-f) - n
                \right] \\
            &= \frac{1}{2(1-\eta)\sigma_\xi^2}\norm{\hat{f}-f}_2^2 \tag{*}.} 
            
            We can then apply lemma \ref{lem:cond_kl} to demonstrate that, if we can find an
        upper bound on the 2-norm between true and approximate function evaluations on some $x$ data, we will be
        able to correspondingly bound the KL-divergence between the marginal distributions of the noise-corrupted
        samples and finally the TV via the inequality \eqref{pinsker}.

\section{Random Fourier Features}
\label{appx:rff}
\begin{algorithm}[H]
Let \(M\) denote memory usage for each line. \\
Define a rule to set a random seed for each \(j\) in \(1..D\) to ensure the same vector \(\omega_j\) is used for each row of \(X\). \\
\For{$i:n$}{
\begin{algorithmic}
    \State Sample \((x_i^{(1)},\dots,x_i^{(d)}) \sim \mathbb{P}_X\) \Comment{\(M=\order{d}\)}
    \State \(\hat{f}_i \gets 0\) 
    \end{algorithmic}
    \For{$j:D$}{
    \begin{algorithmic}
            \State Sample \((\omega_j^{(1)},\dots, \omega_j^{(d)}) \sim \mathbb{P}_\Omega\) \Comment{\(\order{1} \leq M \leq \order{d^2}\)}
            \State Compute \(z_j(x_i) = g(x_i^T\omega_j)\) \Comment{\(M=\order{1}\)}
            \State Sample \(w_j \sim \normdist{0}{1}\) \Comment{\(M=\order{1}\)}
            \State \(\hat{f}_i \gets \hat{f}_i + z_j(x_i)w_j\) \Comment{\(M=\order{1}\)}
    \end{algorithmic}
    }
}
\caption{Memory-efficient procedure to generate RFF samples.}
\label{alg:rff_sampling}
\end{algorithm}
Algorithm \ref{alg:rff_sampling} represents an extreme example of how we can trade-off sequential time complexity in the RFF procedure to produce an exceptionally memory-efficient method of at worst \(\order{d^2}\) (in the most general case where a \(d\times d\) matrix is required to sample from \(\mathbb{P}_\Omega\)). We would not recommend this algorithm for a practical implementation, but rather as an illustration of how far the principle can be pushed: massively distributing work amongst many processors and opting to write each sample to disk to avoid storing the full length \(n\) vector output. 

\label{appx:rff_proof}

    \begin{proof}[Proof of \ref{lem:rff}]
         Define the matrix of differences $E_{ij} = z(x_i)^Tz(x_j)-k(x_i,x_j)$.
         \eq{
         |E_{ij}|<\varepsilon/n &\implies \underbrace{\sum_{ij}E_{ij}^2}_{=\norm{E}_F^2}<\varepsilon^2\implies \norm{E}_F < \varepsilon
         }
         which implies
         \eq{
             \Pr\left(\abs{E_{ij}}<\frac{\varepsilon}{n} \quad \forall i,j\right) &= \Pr\left(\abs{E_{ij}}^2<\frac{\varepsilon^2}{n^2} \quad
             \forall i,j\right) \\
             &\leq \Pr\left(\sum_{ij}\abs{E_{ij}}^2<\sum_{ij}\frac{\varepsilon^2}{n^2}\right) \\
             &= \Pr\left(\sum_{ij}\abs{E_{ij}}^2< \varepsilon^2\right). }

         At a particular pair of locations $(x,x')$ which correspond to an element of the kernel matrix, we can use the
         bounds on the (vector) functions $z_j(x) = \sqrt{\frac{2}{D}}(\sin(\omega_j^Tx),\cos(\omega_j^Tx))^T$ (as suggested in \cite{Sutherland2015}) and Hoeffding's inequality to get a
         probabilistic tail bound on the error of an element of the approximate kernel matrix: \eq{ S_{D/2} &= \sum_{i=1}^{D/2}
             z_i(x)^Tz_i(x') - k(x,x') \\
             -2/D &\leq z_i(x)^Tz_i(x') \leq 2/D \\
             p = \Pr\left(|S_{D/2}|\geq \frac{\varepsilon}{n}\right) &\leq
             2\exp\left(\frac{-2\varepsilon^2}{n^2\sum_{i=1}^{D/2}(2/D--2/D)^2}\right) \\
             &= 2\exp(-D\varepsilon^2/4n^2) .}

         Now we apply the union bound, assuming that the locations are relatively uncorrelated. Note that $E$ is symmetric
         and we can ensure the diagonal elements are 0 so that we only need to bound $\half n(n-1)$ elements: \eq{ q =
             \Pr\left(\bigcup_{ij}\left\{\abs{E_{ij}}\geq \frac{\varepsilon}{n}\right\}\right) &\leq p\cdot \half n(n-1) =
             n(n-1)\exp(-D\varepsilon^2/4n^2) \\
             &\leq n^2\exp\left(-\frac{D\varepsilon^2}{4n^2}\right)\\
             \Pr\left(\abs{E_{ij}}<\frac{\varepsilon}{n} \quad \forall i,j\right) &= 1-q. }

         If we now state that we wish to choose a number of RFF $D$
         s.t. all elements of the error matrix are less than $\varepsilon/n$ with probability $1-\delta$ then we can
         rearrange the final expression above with $q=\delta$ to find:
         \eq{ D &\geq
             8\log\left(\frac{n}{\sqrt{\delta}}\right)\frac{n^2}{\varepsilon^2} \\
             &\implies \norm{E}_F^2 < \varepsilon^2 \\
             &\implies \KL <\frac{\varepsilon^2}{4\sigma_\xi^4} \quad (\ref{lem:kl_frob})\\
             &\implies \TV < \frac{\varepsilon}{\sqrt{8}\sigma_\xi^2} \quad \eqref{pinsker}.
         }
         To complete the proof we set \(\eps=\sqrt{8}\sigma_\xi^2\varepsilon\) so that \(\TV<\eps\).
     \end{proof}

\section{Contour Integral Quadrature}
\label{appx:ciq_summary}

\subsection{Summary of the CIQ method}
Here we give a brief summary of the CIQ method, as discussed in \cite{Davies2005}, \cite{Hale2008} and \cite{Pleiss2020}. We neglect much of the intricacies, which can be found in the aforementioned sources.
The CIQ method relies on a numerical quadrature approximation of the matrix version of Cauchy's integral theorem, for some square matrix \(A\):
\[f(A) = \frac{1}{2\pi\it{i}}\int_\Gamma f(z)(zI-A)^{-1}dz. \]
As usual in complex analysis \(\Gamma\) is a closed anticlockwise contour over which \(f\) is analytic. 

In our case, we want to use \(f(z)=z^\half\) so that
\[ A^\half = \frac{1}{2\pi\it{i}}\int_\Gamma z^\half(zI-A)^{-1}dz.\]
Section 4 of \cite{Hale2008} pays particular attention to this case and makes a change of variables \(w=z^\half\) to find
\eq{
A^\half &= \frac{A}{\pi\it{i}}\int_{\Gamma_w}(w^2I-A)^{-1}dw\\
&= \frac{\it{i}A}{\pi}\int_{-\it{i}\infty}^{\it{i}\infty}(w^2I-A)^{-1}dw
}

This expression is then approximated using a trapezoid rule with \(Q\) quadrature points:
\[\hat{A}_Q^\half = \frac{-2K'm^\half A}{\pi Q}\sum_{q=1}^Q(w(t_q)^2I-A)^{-1}\mathrm{cn}(t_q)\mathrm{dn}(t_q)
\]
where \(\mathrm{cn},\mathrm{dn}\) are Jacobi elliptic functions in standard notation.

To compute matrix-vector products, as we intend to, note that Cauchy's integral formula can be adapted straightforwardly to
\[ f(A)b = \frac{1}{2\pi\it{i}}\int_\Gamma f(z)(zI-A)^{-1}b dz.\]

Although these expressions are sufficient to calculate the desired approximations, we refer the reader to \cite{Pleiss2020} for significantly more detail on the practicalities of an efficient implementation.

 \label{appx:ciq_proof}

 \subsection{Proof of bounds for CIQ parameters (\ref{lem:ciq_bounds})}
 \begin{proof}

     From \cite{Pleiss2020} we get the following expression for the error when using CIQ to approximately sample from a
     multivariate normal:

     \numeq{
     \label{eq:epsJ}
         \varepsilon_J = \norm{a_J-K^\half u}_2 &\leq
         \underbrace{\order{\exp\left(-\frac{2Q\pi^2}{\log\kappa+3}\right)}}_{\varepsilon_Q} +
         \underbrace{\frac{2Q\log(5\sqrt{\kappa})\kappa\sqrt{\lambda_n}}{\pi}\left(\frac{\sqrt{\kappa}-1}{\sqrt{\kappa}+1}\right)^{J-1}}_{B(Q,J)}\norm{u}_2. 
     }
     Clearly the most important factor in the number of iterations, $J$, required for the algorithm to satisfy some
     prespecified degree of accuracy $\varepsilon_J$ for a sample of size $n$ depends on the condition number
     $\kappa$ of the kernel matrix. Thus we need a bound on the condition number, which we can find by the following
     argument.

     Using the spectral condition number $\kappa = \frac{\lambda_1}{\lambda_n}$, where we have ordered the
     eigenvalues such that $\lambda_1 \geq \lambda_2 \geq \dots \lambda_n$, we first bound the largest eigenvalue by
     noting that for a kernel $K$ with kernel-scale $\sigma_f^2=1$, $\Tr K = n$ and where $\Tr K = \sum_{i=1}^n\lambda_i$, we
     see that $\lambda_1 \leq n$. For a kernel with added uncorrelated noise of the form $K_{\eta\xi}=K+\eta\sigma_\xi^2 I_n$, we
     similarly see that $\lambda_1(K_{\eta\xi}) \leq n + \eta\sigma_\xi^2$. We include this latter case as it also allows us to bound
     the minimum eigenvalue as $\lambda_n(K_{\eta\xi}) \geq \eta\sigma_\xi^2$ and so
     \eq{
         \kappa(K_{\eta\xi}) &\leq \frac{n}{\eta\sigma_\xi^2} + 1.
     }

     Now, starting from the required bound on the TV-distance we can obtain lower bounds for the CIQ fidelity parameters \(Q,J\) to ensure that \(Q\) is such that \(\varepsilon_Q < \delta_Q\) for some free parameter \(\delta_Q\) and \(\TV(q,p)\leq\eps\):
     
     \eq{
     \TV(q,p) &\leq \sqrt{\half\KL(q,p)} \leq \frac{1}{\sqrt{2}}\sqrt{\E_u\left[\KL(q,p\given u)\right]} \\
     &= \half\sqrt{\E_u\left[\frac{1}{(1-\eta)\sigma_\xi^2}\norm{f-\hat{f}}^2_2\right]} \\
     &\leq \frac{1}{2\sqrt{1-\eta}\sigma_\xi}\sqrt{\E_u\left[\varepsilon_Q + B(Q,J)\norm{u}_2^2\right]} \\
     &= \frac{1}{2\sqrt{1-\eta}\sigma_\xi}\sqrt{\varepsilon_Q^2 + B(Q,J)^2\E_u\norm{u}_2^2 + \varepsilon_QB(Q,J)\E_u\norm{u}_2} \\
     &\leq \frac{1}{2\sqrt{1-\eta}\sigma_\xi}(\delta_Q + \sqrt{n}B(Q,J)) \leq \half\eps.
     }
     
     Where we have used Pinsker's inequality \eqref{pinsker}, lemma \ref{lem:cond_kl}, \eqref{eq:cond_kl_gauss}, \eqref{eq:epsJ} and that $u\sim \normdist{0}{I_n}$ and hence $\norm{u}_2\sim\chi_n$ (note not \(\chi^2\)). Thus we can explicitly find the mean (and variance) (\cite{chidist}) and hence a bound on \(\E_u\norm{u}_2\). Here we will assume $n$ is large and use the asymptotic expansion in \cite{Laforgia2012} to simplify the result:
        \eq{
             \E[\norm{u}_2] &= \sqrt{2}\frac{\Gamma(\half(n+1))}{\Gamma(n/2)} \\
             &= \sqrt{n}(1-\frac{1}{4n} + \order{n^{-2}}) \leq \sqrt{n}. 
         }
         
         From here we can write 
         \numeq{
         \frac{\eps\sigma_\xi\sqrt{1-\eta}-\delta_Q}{\sqrt{n}} &\geq B(Q,J) 
         \label{eq:BQJ}
         }
         and rearrange for \(J\) (once we have bounded \(Q\) in terms of our free parameter \(\delta_Q\)). Note that \eqref{eq:BQJ} determines an upper bound on \(\delta_Q\) to guarantee that the LHS is positive.
         
         In the next sections we make extensive use of the following relationships (where \(\kappa\leq 1+\zeta\) with \(\zeta \gg 1\)):
             \numeq{
             \label{eq:sqrtkappa}
             \sqrt{\kappa} \leq \sqrt{1+\zeta} = \sqrt{\zeta}(1+\zeta^{-1})^\half \leq \sqrt{\zeta} + \frac{1}{2\sqrt{\zeta}} 
             }
             and
             \numeq{
             \label{eq:logkappa}
              \log\kappa = \log\zeta + \log(1+\zeta^{-1}) \leq \log\zeta + \zeta^{-1}
             }
             to see that
             \numeq{
             \label{eq:sqrtkappalogkappa}
             \sqrt{\kappa}\log\kappa\leq \sqrt{\zeta}\log\zeta + \frac{1}{2\sqrt{\zeta}}\log\zeta + \frac{1}{\sqrt{\zeta}} + \frac{1}{2\zeta^{3/2}}.
             }
             Similarly,
             \eq{
             \log\log\frac{x}{y} &= \log(\log x - \log y) = \log\log x + \log\left(1-\frac{\log y}{\log x}\right) \\
             &=\log\log x -\frac{\log y}{\log x} + \order{(\log x)^{-2}}
             }
             giving
             \numeq{\label{eq:loglogxy1}
             \log\log\frac{x}{y} &= \log\log x + \order{(\log x)^{-1}}
             }
             provided \(x \gg y\), and
             \numeq{\label{eq:loglogxy2}
             \log(\log x+y) &= \log\log x +\order{(\log x)^{-1}}
             }
             if additionally \(y < \log x\).

     \subsubsection{Bounding the number of Quadrature Points}
         To ensure \(\varepsilon_Q\) does not exceed \(\delta_Q\) we will constrain \(Q\) as follows:
         \numeq{
         \label{eq:Qkappa}
             Q &\geq \left(\log\kappa + 3\right)(-\log\delta_Q)\frac{1}{2\pi^2}.
        }
        Using \eqref{eq:logkappa} with \(\zeta=\frac{n}{\eta\sigma_\xi^2}\) we can achieve a sufficient \(Q\) by requiring that
        \numeq{
             Q &\geq \frac{1}{2\pi^2}\left(\log\frac{n}{\eta\sigma_\xi^2} + 3 + \order{n^{-1}}\right)(-\log\delta_Q)
             \label{eq:Q}.
         }

     \subsubsection{Bounding the number of msMINRES Iterations}
         Taking \eqref{eq:epsJ} and \eqref{eq:BQJ} we rearrange in terms of $J$ to find
         \numeq{
         \label{eq:J1}
             J &\geq 1 + \frac{1}{\log(\sqrt{\kappa}-1) - \log(\sqrt{\kappa}+1)}\log\left\{\frac{\pi(\eps\sigma_\xi\sqrt{1-\eta} -
                 \delta_Q)}{2Q\sqrt{\lambda_n}\kappa\sqrt{n}(\log(5\sqrt{\kappa}))}\right\}.
         }
         We start by simplifying the prefactor (making use of Taylor expansions):
         \eq{
             \log(\sqrt{\kappa}-1) - \log(\sqrt{\kappa}+1) &= \log(1-1/\sqrt{\kappa}) - \log(1+1/\sqrt{\kappa}) \\
             &= -\frac{2}{\sqrt{\kappa}} - \order{\kappa^{-3/2}} \\
             (\log(\sqrt{\kappa}-1) - \log(\sqrt{\kappa}+1))^{-1} &= -\frac{\sqrt{\kappa}}{2}\left(1-\order{\kappa^{-1}}\right). 
         }
         
         Before we substitute our bound for the condition number, we first give a more general expression. To obtain it we define the RHS of \eqref{eq:Qkappa} to be \(\bar{Q}\) and that \(\log\bar{Q} \leq \log\log\kappa + \log(-\log\delta_Q)+\order{(\log\kappa)^{-1}}\) using \eqref{eq:loglogxy2}. Hence,
         \numeq{
            J &\geq 1 + \frac{\sqrt{\kappa}}{2}\left[\log(\kappa\sigma_\xi\sqrt{n}) + 2\log\log\kappa - \log(\eps\sigma_\xi\sqrt{1-\eta}-\delta_Q) + C\right]
            \label{eq:Jkappa}
         }
         where \(C\) is a pseudo-constant that contains constants, decaying functions of \(\kappa\) and similarly `negligible' terms (such as \(\log(-\log\delta_Q)\) and the small error term \(\abs{\sqrt{\lambda_n}-\sqrt{\eta}\sigma_\xi}\)). Note that the RHS here is larger than the RHS of \eqref{eq:J1} so that it is a more conservative bound.
         
         Now using our bounds for \(\kappa\), \(\sqrt{\kappa}\) \eqref{eq:sqrtkappa} and \(\log\kappa\) \eqref{eq:logkappa} (with \(\zeta=\frac{n}{\eta\sigma_\xi^2}\)) we see that (denoting \(\tilde{J}\) as the RHS of \eqref{eq:Jkappa})
         \eq{
         \tilde{J} &\leq 1 + \frac{\sqrt{n}}{2\sqrt{\eta}\sigma_\xi}\left\{\log\left(\left(\frac{n}{\eta\sigma_\xi^2}+1\right)\sigma_\xi\sqrt{n}\right) + 2\log\log\left(\frac{n}{\eta\sigma_\xi^2}+1\right) -\log(\pi(\eps\sigma_\xi\sqrt{1-\eta}-\delta_Q)) + C' \right\} \\
         &\leq 1 + \frac{\sqrt{n}}{2\sqrt{\eta}\sigma_\xi}\left\{\log\left(\frac{n^{3/2}}{\eta\sigma_\xi}\right) + 2\log\log\frac{n}{\eta\sigma_\xi^2} -\log(\pi(\eps\sigma_\xi\sqrt{1-\eta}-\delta_Q)) + C''\right\} \\
         &\leq 1 + \frac{\sqrt{n}}{2\sqrt{\eta}\sigma_\xi}\left\{\log n^{3/2} -\log(\pi(\eps\sigma_\xi\sqrt{1-\eta}-\delta_Q)) + 2\log\log n + C'''\right\} \\
         &= \ordtilde{\frac{\sqrt{n}}{\sqrt{\eta}\sigma_\xi}\log\frac{n}{\sigma_\xi(\eps\sigma_\xi\sqrt{1-\eta}-\delta_Q)}}
         }
         where we have again used \eqref{eq:logkappa} and \eqref{eq:loglogxy1} and absorbed the constant and decaying terms into the sequence of `constants' \(C',C'',C'''\).
         
         Having upper bounded \(\tilde{J}\), if we now use this as a lower bound for \(J\) then we have a sufficient condition to satisfy our TV requirement, i.e.
         \numeq{
          J & \geq \ordtilde{\frac{\sqrt{n}}{\sqrt{\eta}\sigma_\xi}\log\frac{n}{\sigma_\xi(\eps\sigma_\xi\sqrt{1-\eta}-\delta_Q)}}
         }
         where we use \(\ordtilde{\cdot}\) to mean ignoring \(\log\log\) terms.
         

\end{proof}

         
         

     \subsubsection{Preconditioning}

         \begin{proof}[Proof of \ref{lem:nystrom_precond}]
             If we take the rank-$k$ Nystr\"{o}m approximation ($\tilde{K}$) as our preconditioner, with cost $\order{Nk^2}$, then Corollary 4.10 of \cite{Shabat2019} shows:
             \numeq{
             \label{eq:cond1}
                 \tilde{\kappa} = \cond[(\tilde{K} + \eta\sigma_\xi^2 I)^{-1}(K+\eta\sigma_\xi^2 I)] &\leq 1 + \frac{2\lambda_{k+1}(K)\sqrt{4k(n-k)+1}}{\eta\sigma_\xi^2}.
             }
             (Henceforth we use \(\lambda_k\) to denote \(\lambda_k(K)\).) To satisfy our cost requirement we set $k=\lfloor{\sqrt{n}}\rfloor$ and since 
             \eq{
             (4\sqrt{n}(n-\sqrt{n}) + 1)^\half &= (1+4n^{3/2}(1-1/\sqrt{n}))^\half\\
             &= 2n^{3/4}(1-1/\sqrt{n})^\half\left(1+\frac{1}{8}n^{-3/2}(1-1/\sqrt{n})^{-1} + \order{n^{-3}}\right) \\
             &=2n^{3/4}\left(1-\frac{1}{2\sqrt{n}} - \frac{1}{8n} - \frac{1}{16n^{3/2}} + \order{n^{-2}}\right)\left(1+\frac{1}{8n^{3/2}} +
             \order{n^{-2}}\right) \\
             &\leq 2n^{3/4}
             }
             we have
             \numeq{
             \label{eq:cond2}
             \tilde{\kappa} &\leq 1+\frac{4\lambda_{k+1}}{\eta\sigma_\xi^2}n^{3/4}
             }
             which we write as \(\tilde{\kappa} \leq 1+\zeta\), as before, but now with \(\zeta=\frac{4\lambda_{k+1}}{\eta\sigma_\xi^2}n^{3/4}\).

             With this value of \(\zeta\) we again make use of \eqref{eq:sqrtkappa}, \eqref{eq:logkappa} and \eqref{eq:sqrtkappalogkappa} to show that
             \numeq{
             \label{eq:cond3}
                \sqrt{\tilde{\kappa}}\log\tilde{\kappa} &\leq \frac{2\sqrt{\lambda_{k+1}}}{\sqrt{\eta}\sigma_\xi}n^{3/8}\log\left(\frac{4\lambda_{k+1}}{\eta\sigma_\xi^2}n^{3/4}\right) + \order{n^{-3/8}\log n} .
             }
             
             Similarly, we have
             \numeq{
             \label{eq:cond4}
             \sqrt{\tilde{\kappa}} &\leq \sqrt{\zeta} + \half\zeta^{-\half} = \frac{2\sqrt{\lambda_{k+1}}}{\sqrt{\eta}\sigma_\xi}n^{3/8} + \order{n^{-3/8}}.
             }
             
             We finish the proof by rewriting \eqref{eq:Jkappa} in the form
             \eq{
             J &\geq 1 + \frac{\sqrt{\tilde{\kappa}}}{2}\left(\log(\tilde{\kappa}\sqrt{n}\sigma_\xi) - \log(\eps\sigma_\xi\sqrt{1-\eta}-\delta_Q) + \tilde{C}\right)
             }
             and inserting \eqref{eq:cond3} (updating \(\tilde{C}\rightarrow\tilde{C}'\) to absorb additional approximately negligible terms).
         \end{proof}

         
         \begin{proof}[Proof of \ref{lem:general_precond}]\label{proof:general_precond}
             \cite{Belkin2018} shows us that for sufficiently smooth radial kernels, the \(k^{th}\) matrix eigenvalue is given by (e.g. Gaussian, Cauchy) \eq{
             \lambda_k &\lesssim n\sqrt{\varphi}c_2\e^{-c_1k^{1/d}}
             }
             for $c_1,c_2>0$, $x \in \mathbb{R}^d$, $\varphi=\sup_{x\in\Omega}k(x,x)$.
             For us, $\varphi=\sigma_f^2$ and we set $k=\floor{\sqrt{n}}$.
             
             We can use this to see
             \eq{
             \sqrt{\lambda_{k+1}}n^{3/8} &\leq \sqrt{\lambda_k}n^{3/8} \leq \sqrt{c_2\sigma_f}n^{7/8}\e^{-\frac{c_1}{2}n^{1/d}} \leq \sqrt{c_2\sigma_f}n^{7/8}
             }
             and hence use this with lemma \ref{lem:nystrom_precond} (for \ref{itm:modn})
             \eq{
             J &\geq 1 + \frac{\sqrt{c_2\sigma_f}}{\sqrt{\eta}\sigma_\xi}n^{7/8}\left(\frac{5}{4}\log n - \log(\eps\sigma_\xi\sqrt{1-\eta}-\delta_Q) + \tilde{C}'\right) \\
             &= 1 + \order{n^{7/8}\log n}.
             }

             For modest \(n\) such that \(0<\frac{3}{8}\log n - \frac{c_1}{2}n^\frac{1}{d} < 1\) (\ref{itm:bign}),
             \eq{
             \exp\left[\frac{7}{8}\log n - \frac{c_1}{2}n^\frac{1}{d}\right] &\leq 1 + \frac{7}{8}\log n 
             }
             which we combine with \ref{lem:nystrom_precond} to obtain 
             \eq{
             J &\geq 1 + \frac{\sqrt{c_2\sigma_f}}{\sqrt{\eta}\sigma_\xi}\left[\frac{31}{32}(\log n)^2 + \frac{5}{4}\log n - \log(\eps\sigma_\xi\sqrt{1-\eta} - \delta_Q) + \tilde{C}''\right] \\
             &= \order{(\log n)^2}
             }
             where we have again absorbed small constant, decaying and \(\log\log\) terms into \(\tilde{C}''\).
             
             At sufficiently large \(n\) (\ref{itm:asympn}), we can write
             \eq{
             \sqrt{\lambda_k}n^{3/8}\log n &\leq \sqrt{c_2\sigma_f}\exp\left[\frac{7}{8}\log n + \log\log n - \frac{c_1}{2}n^\frac{1}{d}\right] \\
             & \leq \sqrt{c_2\sigma_f}
             }
             since for all finite $d$ the $n^\frac{1}{d}$ term grows faster in $n$ than $\log n$ (and definitely \(\log\log n\)), the exponent will become negative at large $n$. This gives us
             \[
                J \geq 1 + \order{1}.
             \]
         \end{proof}
         
        \begin{remark}
         Note that the constant terms hidden in the \(\order{\cdot}\) notation are largely consistent for all \(\gamma\) and generally less than \(\order{\sigma_\xi^{-1}}\).
         \end{remark}
        
        \begin{remark}
                For the special case of the RBF kernel we can exploit the precise expressions for the eigenvalues to obtain more specific bounds on \(J\), but we skip these details here.
        \end{remark}

\section{Conjugate Gradients (CG)}
\label{appx:cg}
In addition to the sampling methods described in this paper we would also like to acknowledge that the conjugate gradients algorithm can also be adapted to facilitate approximate sampling as, for example, in \cite{Parker2012}. Since the computational complexity of such a procedure is in general \(\order{n^2k}\) where \(k\) is the number of iterations employed by the conjugate gradient algorithm, in order for this to be competitive, we require \(k\leq\order{\sqrt{n}\log n}\). Within the time-frame of this paper, we have been unable to investigate this option fully, but we believe that in cases where the Gram matrix can be well-approximated by a low-rank matrix (e.g. at large lengthscales), CG sampling is likely to be a promising approach.

\section{Summary of algorithms}
\label{appx:results}
    \begin{table}[h!]
        \centering
        \begin{tabular}{lll}
        \hline
        \textbf{Method} & \textbf{Time} & \textbf{Space}\\
           \hline
                Cholesky    &   \(\order{n^3}\) & \(\order{n^2}\) \\
                RFF         &   \(\order{n^3\log n}\) & \(\order{n}\)\\
                CIQ         &   \(\order{n^{5/2}\log n}\) & \(\order{n\log n}\) \\
                PCIQ         &   \(\order{n^{2.375}\log n}\) & \(\order{n\log n}\) \\
            \hline
        \end{tabular}
        \caption{Time and space complexity of competing methods of generating draws from a GP. P=with preconditioning. \label{appx:table}}
    \end{table} 
    
In addition to the algorithms in table \ref{appx:table} we point out that the RFF method is highly parallelisable and the `wall-clock' time could be significantly reduced from that listed, with a large enough supply of processors; but it is beyond the scope of this paper to delve into the details of such an implementation. Finally, the PCIQ entry does not include the gains observed as \(n\) becomes very large (as outlined in \ref{lem:general_precond}). 

\section{Implementation}
\label{appx:implementation}
    We implemented the RFF sampling procedure in NumPy and made use of the GPyTorch\footnote{\url{github.com/cornellius-gp/gpytorch/}} (\cite{Gardner2018}) library to facilitate CIQ. We wrapped a SciPy implementation of interpolative decomposition to integrate with GPyTorch for preconditioning with the CIQ method. These will be made available on our GitHub (\url{github.com/ant-stephenson/gpsampler}) and can be installed as a Python library.
    
    To run our experiments we made use of the HPC system BluePebble at the University of Bristol, using nodes with 12 CPUs with 15GB of memory each and allocating a maximum of 200 hours per experiment. This was sufficient for our purposes and this much parallel compute was only necessary to facillitate the running of 1000 repeats per experiment.

\section{Experiments}
\label{appx:experiments}
    We present results from some empirical experiments running hypothesis tests in section \ref{sec:experiments}. We chose to use hypothesis testing rather than directly implement the Bayesian decision procedure to avoid considerable coding effort and compute resource whilst still demonstrating our main points.
    
    A more complete description of the method used is the following:
    
    \begin{algorithm}[H]
    \KwIn{A set of adjustable experiment parameters \(\theta\); \(N\) the number of repeat experiments per parameter set.}
    \KwOut{Hypothesis test rejection rate, \(r\).}
    \begin{algorithmic}
         \State \(r\gets 0\) 
    \end{algorithmic}
         \For{$i:N$}{
            \begin{algorithmic}[1]
                 \State Generate sample \(\hat{y}\) of length \(n\) using method \(M(\theta)\). 
                 \State Whiten the sample using a Cholesky decomposition of the true kernel matrix, i.e. \(\hat{z}=L^{-1}\hat{y}\) for \(K_\xi=LL^T\).
                 \State Run a (Cramér-von Mises) hypothesis test to determine whether the whitened sample \(\hat{z}\) is consistent with an i.i.d draw from a standard normal distribution. Record the test result \(t \in \{0,1\}\) at a predetermined significance level \(\alpha\).
                 \State \(r \gets r + t\)
            \end{algorithmic}
            }
            \begin{algorithmic}
            \State \(r \gets r/N\)
            \end{algorithmic}
        
         \caption{Procedure used to test sample quality.}
         \label{alg:hyp_test}
     \end{algorithm}
     
     In addition to figure \ref{fig:results}, figure \ref{fig:results_more} shows the results of further experiments run at larger lengthscales to assess how performance of the CIQ method degrades as condition number becomes more extreme. We note that apart from an increase in variance (which is expected), the results appear to be quite stable, oscillating around the blue line and mostly contained between the green lines. The blue line represents the rejection rate we expect under the null hypothesis (that the distributions are the same) and is thus the rate we expect to achieve at convergence. The green lines are given by the 95\% confidence intervals for a large-sample of Bernoulli trials at the converged rate.
    \begin{figure}
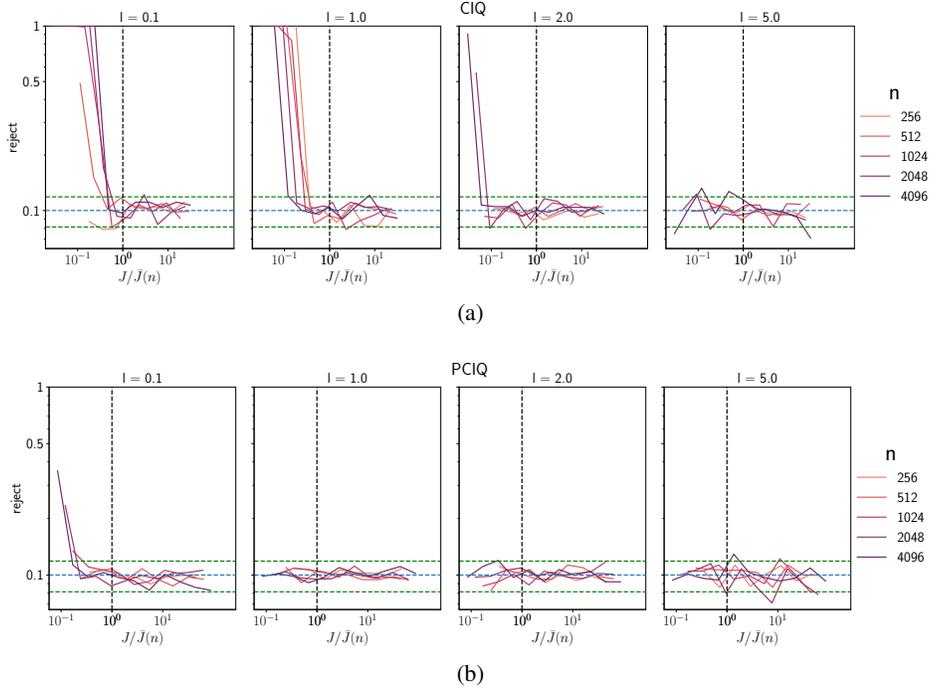

        \centering
        \subfloat[]{\label{appx_main:b}\includegraphics[clip,scale=0.3]{\detokenize{logreject-logD_byN_ciq_1_2703849_rescaled_3}}}\\
        \centering
        \subfloat[]{\label{appx_main:c}\includegraphics[scale=0.3]{\detokenize{logreject-logD_byN_ciq_1_2686645_rescaled_2}}}
        \caption{Rejection rate convergence with size of fidelity parameter as before, with additional plots at more extreme lengthscales.}
        \label{fig:results_more}
\end{figure}
    
    
    On the results themselves, we note that the PCIQ method appears to converge more slowly for \(l=0.1\) than for \(l=1.0\), an initially surprising result. At first consideration, we expect the kernel matrix for the former to be closer to the identity and thus have a smaller condition number. This is true, but conversely, the effectiveness of the preconditioning step is hindered by the fact that at small lengthscales the matrix will also be full rank and thus to well-approximate the inverse we are likely to need a higher rank approximation. To test this hypothesis we ran a simulation of our (rank-\(\sqrt{n}\)) preconditioner acting on a series of random kernel matrices with varying lengthscales. Figure \ref{fig:precond_effect} shows the results from this simulation from which can be seen an apparent peak in the vicinity of \(l=0.1\), in particular for the \(n \in \{2000,5000\}\) cases, which replicates our previous findings. It can be shown that the worst-case lengthscale gets smaller as \(n\) increases, tending to 0 in the limit as \(n\rightarrow\infty\). 
    
        \begin{figure}
        \centering
        \includegraphics[scale=0.5]{\detokenize{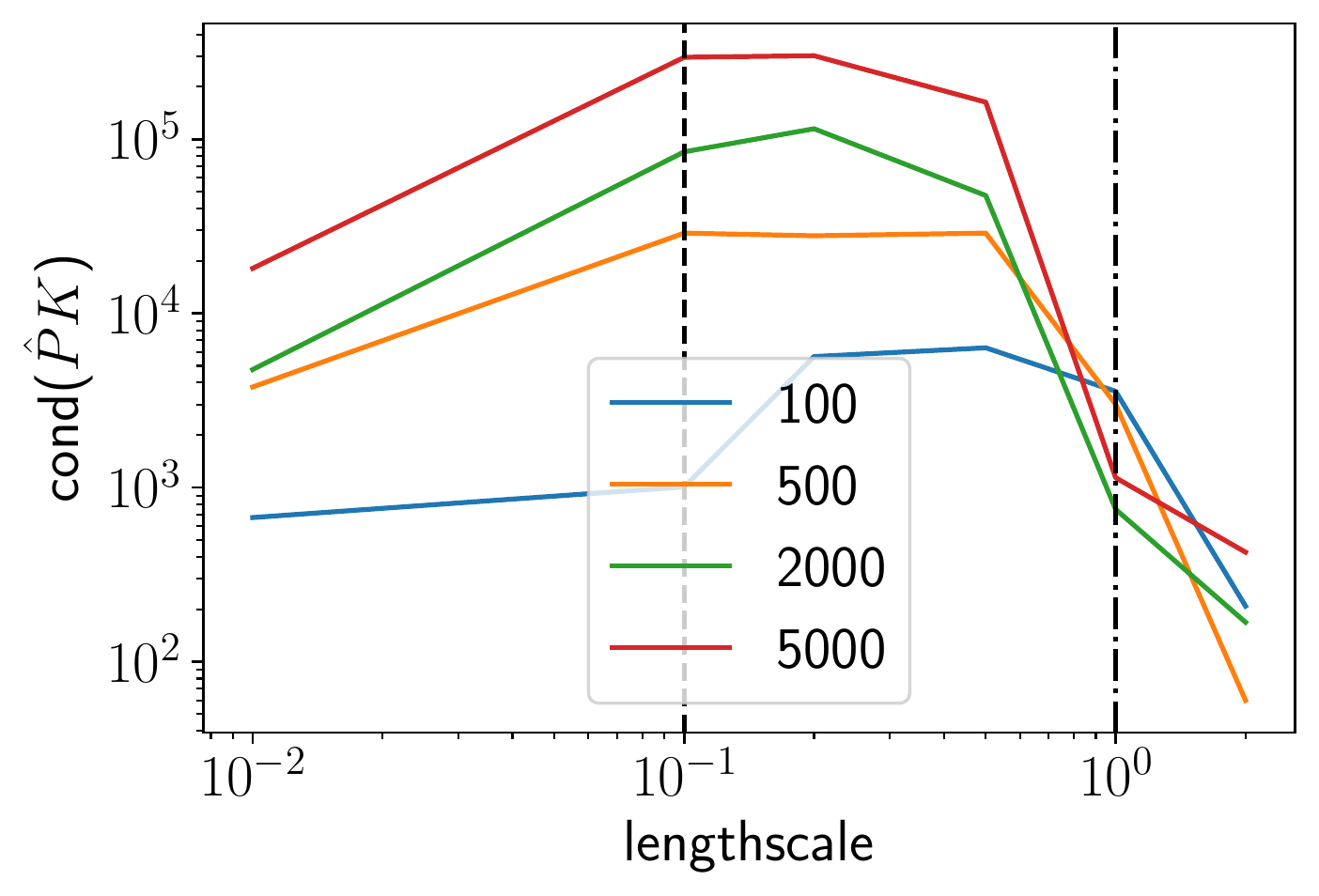}}
        \caption{Preconditioner effectiveness as a function of lengthscale. Different colour lines represent different sample sizes. For this simulation we used parameter values consistent with our previous experiments: \((d, \sigma_\xi^2,\sigma_f^2)=(2, 0.001,1.0)\). The (first) black dashed line is at \(l=0.1\) and the second dot-dash line is at \(l=1.0\). \(\hat{P}\) represents the preconditioner approximation to the kernel matrix inverse.}
        \label{fig:precond_effect}
\end{figure}

\section{Further work}
\label{appx:further}
As noted in \ref{sec:results}, we believe the bounds on at least the RFF method can be improved upon. Additionally, more extensive experiments over a wider hyperparameter space and more general kernel functions, as well as utilising the Bayesian decision process outlined in the paper, would provide more thorough support to the arguments. We also acknowledge the vast amount of literature dedicated to efficient implementations of various linear algebra routines (e.g. SVD and Cholesky) that could be utilised for our purposes, albeit with considerable effort to derive similar theoretical guarantees.

In running the experiments we made extensive use of the GPyTorch machinery, pushing it beyond its intended use; we wish to acknowledge that our application of CIQ is a `misuse' of the GPyTorch implementation as it was originally designed. As a result, we believe it would be beneficial to adjust it to be more compatible with this application, in order to facilitate further adoption of synthetic data for GP evaluation.

\end{document}